\documentclass[sigconf,preprint]{acmart}
\usepackage{microtype}
\usepackage{graphicx}
\usepackage{amsmath}
\usepackage{booktabs}
\usepackage{enumitem}
\usepackage{caption}
\usepackage{subfig}
\usepackage{epsfig}
\usepackage{comment}
\usepackage{xspace}
\usepackage{amsfonts }
\usepackage{mathtools}
\usepackage{multirow}
\usepackage{hyperref}
\usepackage{algorithm}
\usepackage{algorithmic}
\usepackage{float}

\newcommand{\proj}{\textsc{TiFL}\xspace}

\begin{document}

\title{{\proj}: A Tier-based Federated Learning System}

\author{Zheng Chai*} 
\email{zchai2@gmu.edu}
\affiliation{
  \institution{George Mason University}
}
\author{ Ahsan Ali* }
\email{aali@nevada.unr.edu}
\affiliation{
  \institution{University of Nevada, Reno}
}

\author{Syed Zawad*}
\email{szawad@nevada.unr.edu}
\affiliation{
  \institution{University of Nevada, Reno}
}
\author{Stacey Truex}
\email{staceytruex@gatech.edu}
\affiliation{
  \institution{Georgia Institute of Technology}
}

\author{ Ali Anwar }
\email{ali.anwar2@ibm.com}
\affiliation{
  \institution{IBM Research - Almaden}
}
\author{Nathalie Baracaldo}
\email{baracald@us.ibm.com}
\affiliation{
  \institution{IBM Research - Almaden}
}
\author{Yi Zhou}
\email{yi.zhou@ibm.com}
\affiliation{
  \institution{IBM Research - Almaden}
}
\author{Heiko Ludwig}
\email{hludwig@us.ibm.com}
\affiliation{
  \institution{IBM Research - Almaden}
}
\author{Feng Yan}
\email{fyan@unr.edu}
\affiliation{
  \institution{University of Nevada, Reno}
}
\author{Yue Cheng}
\email{yuecheng@gmu.edu}
\affiliation{
  \institution{George Mason University}
}

\renewcommand{\shortauthors}{Paper \#147}

\settopmatter{printacmref=false} 
\renewcommand\footnotetextcopyrightpermission[1]{} 
\pagestyle{plain} 

\begin{abstract}
Federated Learning (FL) enables learning a shared model across many clients without violating the privacy requirements. One of the key attributes in FL is the heterogeneity that exists in both resource and data due to the differences in computation and communication capacity, as well as the quantity and content of data among different clients. 
We conduct a case study to show that heterogeneity in resource and data has a significant impact on training time and model accuracy in conventional FL systems. 
To this end, we propose \proj, a Tier-based Federated Learning System, which divides clients into tiers based on their training performance and selects clients from the same tier in each training round to mitigate the straggler problem caused by heterogeneity in resource and data quantity. To further tame the heterogeneity caused by non-IID (Independent and Identical Distribution) data and resources, \proj employs an
\textit{adaptive} tier selection approach to update the tiering on-the-fly based on the observed training performance and accuracy over time.
We prototype \proj in a FL testbed following Google's FL architecture and evaluate it using popular benchmarks and the state-of-the-art FL benchmark LEAF. Experimental evaluation shows that \proj outperforms the conventional FL in various heterogeneous conditions. With the proposed adaptive tier selection policy, we demonstrate that \proj achieves much faster training performance while keeping the same (and in some cases - better) test accuracy across the board.

\end{abstract}


\maketitle

\let\thefootnote\relax\footnotetext{* Equal contribution.}

\section{Introduction}
\label{sec:intro}

Modern mobile and IoT devices (such as smart phones, smart wearable devices, smart home devices) are generating massive amount of data every day, which provides opportunities for crafting sophisticated machine learning (ML) models to solve challenging AI tasks~\cite{he2016deep}. 
In conventional high-performance computing (HPC), all the data is collected and centralized in one location and proceed by supercomputers with hundreds to thousands of computing nodes.
However, security and privacy concerns have led to new legislation such as the General Data Protection Regulation (GDPR) ~\cite{tankard2016gdpr} and the Health Insurance Portability and Accountability Act (HIPAA) \cite{o2004health} that prevent transmitting data to a centralized location, thus making conventional high performance computing difficult to be applied for collecting and processing the decentralized data.
Federated Learning (FL) \cite{mcmahan2017learning} shines light on a new emerging high performance computing paradigm by addressing the security and privacy challenges through utilizing decentralized data that is training local models on the local data of each client (data parties) and using a central aggregator to accumulate the learned gradients of local models to train a global model. 
Though the computing resource of individual client may be far less powerful than the computing nodes in conventional supercomputers, the computing power from the massive number of clients can accumulate to form a very powerful ``decentralized virtual supercomputer''.
Federated learning has demonstrated its success in a range of applications.
From consumer-end devices such as GBoard~\cite{hard2018federated, yang2018applied} and keyword spotting~\cite{leroy2019federated} to pharmaceuticals~\cite{EUCordis}, medical research~\cite{courtiol2019deep}, finance~\cite{webank} , and manufacturing~\cite{hao2019efficient}. 
There has also been a rise of FL tools and framework development, such as Tensorflow Federated~\cite{TFL}, LEAF~\cite{caldas2018leaf}, PaddleFL~\cite{paddlefl} and PySyft~\cite{ryffel2018generic} to facilitate these demands.
Depending on the usage scenarios, FL is usually categorized into \emph{cross-silo} FL and \emph{cross-device} FL~\cite{kairouz2019advances}. In cross-device FL, the clients are usually a massive number (e.g., up to $10^{10}$) of mobile or IoT
devices with various computing and communication capacities ~\cite{mcmahan2016communication,kairouz2019advances,konevcny2016federated} while in cross-silo FL, the clients are a small number of organizations with ample computing power and reliable communications~\cite{yang2019federated,kairouz2019advances}.
In this paper, we focus on the cross-device FL (for simplicity, we call it FL in the following), which intrinsically pushes the heterogeneity of computing and communication resources to a level that is rarely found in  datacenter distributed learning and cross-silo FL.
More importantly, the data in FL is also owned by clients where the quantity and content
can be quite different from each other, causing severe heterogeneity in data that usually does not appear in datacenter distributed learning, where data distribution is well controlled.

We first conduct a case study to quantify how data and resource heterogeneity in clients impacts the performance of FL with FedAvg in terms of training performance and model accuracy, and we summarize the key findings below:
(1) training throughput is usually bounded by slow clients (a.k.a. stragglers) with less computational capacity and/or slower communication, which we name as the \textit{resource heterogeneity}. Asynchronous training is often employed to mitigate this problem in datacenter distributed learning, but literature has shown that synchronization is a better approach for secure aggregation~\cite{bonawitz2017practical} and differential privacy \cite{mcmahan2017learning}. Moreover, \textit{FedAvg}~\cite{mcmahan2016communication} has become a common algorithm for FL that uses synchronous approach for training. 
(2) Different clients may train on different quantity of samples per training round and results in different round time that is similar to the straggler effect, which impacts the training time and potentially also the accuracy. We name this observation the \textit{data quantity heterogeneity}.
(3) In datacenter distributed learning, the classes and features of the training data are uniformly distributed among all clients, namely Independent Identical Distribution (IID). However, in FL, the distribution of data classes and features depends on the data owners, thus resulting in a non-uniform data distribution, known as non-Identical Independent Distribution (\textit{non-IID data heterogeneity}). Our experiments show that such heterogeneity can significantly impact the training time and accuracy.

Driven by the above observations, we propose \proj, a Tier-based Federated Learning System. The key idea here is adaptively selecting clients with similar per round training time so that the heterogeneity problem can be mitigated without impacting the model accuracy.
Specifically, we first employ a lightweight profiler to measure the training time of each client and group them into different logical data pools based on the measured latency, called \textit{tiers}.
During each training round, clients are selected uniform randomly from the same tier based on the adaptive client selection algorithm of \proj.
In this way, the heterogeneity problem is mitigated as clients belonging to the same tier have similar training time.
In addition to heterogeneity mitigation, such tiered design and adaptive client selection algorithm also allows controlling the training throughput and accuracy by adjusting the tier selection intelligently, e.g., selecting tiers such that the model accuracy is maintained while prioritizing selection of faster tiers.
We further prove that the tiering method is compatible with privacy-preserving FL.

While \textit{resource heterogeneity} and \textit{data quantity heterogeneity} information can be reflected in the measured training time, the \textit{non-IID data heterogeneity} information is difficult to capture.
This is because any attempt to measure the class and feature distribution violates the privacy-preserving requirements. 
To solve this challenge, \proj offers an \textit{adaptive} client selection algorithm that uses the accuracy as indirect measure to infer the \textit{non-IID data heterogeneity} information and adjust the tiering algorithm on-the-fly to minimize the training time and accuracy impact. Such approach also serves as an online version to be used in an environment where the characteristics of heterogeneity change over time.

We prototype \proj in a FL testbed that follows the architecture design of Google's FL system~\cite{bonawitz2019towards} and perform extensive experimental evaluation to verify its effectiveness and robustness using both the popular ML benchmarks and state-of-the-art FL benchmark LEAF~\cite{caldas2018leaf}.
The experimental results show that in the \textit{resource heterogeneity} case, \proj can improve the training time by a magnitude of 6$\times$ without affecting the accuracy. In the \textit{data quantity heterogeneity} case, a 3$\times$ speedup is observed in training time with comparable accuracy to the conventional FL.
Overall, \proj outperforms the conventional FL with legal parameter number in definition of 3$\times$ improvement in training time and 8\% improvement in accuracy in CIFAR10~\cite{krizhevsky2014cifar} and 3$\times$ improvement in training time using FEMINIST\cite{caldas2018leaf}  under LEAF.
\section{Related Work}
\label{sec:related}
 \textbf{Federated Learning.}
 The most recent research efforts in FL have been focusing on the functionality~\cite{konevcny2016federated}, scalability~\cite{bonawitz2019towards}, privacy~\cite{mcmahan2017learning, bonawitz2017practical}, and tackling heterogeneity~\cite{ghosh2019robust,li2018federated}.
 Existing FL approaches \cite{konevcny2016federated,mcmahan2016communication,caldas2018expanding} do not account for the resource and data heterogeneity,
 mainly focusing on weight and model compression to reduce communication overhead. They are not straggler-aware even though there can be significant latency issues with stragglers. In synchronous FL, a fixed number of clients are queried in each learning epoch to ensure performance and data privacy. Recent synchronous FL algorithms focus on reducing the total training time without considering the straggler clients. For example, \cite{mcmahan2016communication} proposes to reduce network communication costs by performing multiple SGD (stochastic gradient descent) updates locally and batching clients. \cite{konevcny2016federated} reduces communication bandwidth consumption by structured and sketched updates. 
FedCS~\cite{nishio2019client} proposes to solve client selection issue via a deadline-based approach that filters out slowly-responding clients. However, FedCS does not consider how this approach effects the contributing factors of straggler clients in model training. Similarly, \cite{wang2019adaptive} proposes an algorithm for running FL on resource-constrained devices. However, they do not aim to handle straggler clients and treat all clients as resource-constrained. In contrast, we focus on scenarios where resource-constrained devices are paired with more powerful devices to perform FL. 

Most asynchronous FL algorithms do not take into consideration the effects of client drop-outs. For instance, \cite{smith2017federated} provides performance guarantee only for convex loss functions with bounded delay assumption.
Furthermore, the comparison of synchronous and asynchronous methods of distributed gradient descent \cite{bonawitz2019towards} suggest that FL should use the synchronous approach \cite{mcmahan2016communication, bonawitz2017practical}, as it is more secure than the asynchronous approaches.

In the interest of space, for further background information, we recommend readers to read these papers \cite{li2019federated,kairouz2019advances}, where an in-depth discussion is offered for the current challenges and state-of-the-art systems in Federated Learning. 

\textbf{Stragglers in Datacenter Distributed Learning.} Similar to datacenter distributed learning systems, the straggler problem also exists in FL and it is greatly pronounced in the synchronous learning setting. The significantly higher heterogeneity levels in FL makes it more challenging.
\cite{bonawitz2019towards} proposes a simple approach to handle stragglers problem in FL,
where the aggregator selects 130\% of the target number of devices to initially participate, and discards stragglers during training process. With this method, the aggregator can get 30\% tolerance for the stragglers by ignoring the updates from the slower edge devices. However, the 30\% is set arbitrarily which requires further tuning. Furthermore simply dropping the slower clients might exclude certain data distributions available on the slower clients from contributing towards training the global model. FedProx~\cite{li2018federated} also tackles resource and data heterogeneity by making improvements on the FedAvg algorithm. However, they also discard training data to make up for the systems heterogeneity.

\cite{li2018federated} takes into account the resource heterogeneity. However, the proposed approach is mainly focused on only two types of clients - stragglers and non-stragglers. In a real FL environment there is a wide range of heterogeneity levels. The proposed approach performs well in case of high ration of stragglers vs non-stragglers (80-90\%). Moreover, their proposed solution involves partial training on stragglers which can further lead to biasness in trained model and sub-optimal model accuracy as explained in Section \ref{resoure_plus_data}.
\cite{ho2013more} proposes that adding local cache is an efficient and reliable technique to deal with stragglers in datacenter distributed Learning. However, FL is powerless in governing the resources on client side, so it's impractical to implement similar mechanisms in FL.

\cite{ghosh2019robust} proposes a general statistical model for Byzantine machines and clients with data heterogeneity that clusters based on data distribution. While data distribution may cause stragglers, \cite{ghosh2019robust} focuses on grouping edge devices such that their datasets are similar. The authors do not consider the impact of clustering on training time or accuracy.
\cite{harlap2016addressing} propose a novel design named RapidReassignment to handle straggles by specializing work shedding. It uses P2P communication among workers to detect slowed workers, performs work re-assignment, and exploits iteration knowledge to further reduce how much data needs to be preloaded on helpers. However, as stated in Section~\ref{sec:intro}, migrating a user's private data to other unknown users' devices is strictly restricted in FL.
An analogical approach named SpecSync is proposed in \cite{zhang2018stay}, where each worker speculates about the parameter updates from others, and if necessary, it aborts the ongoing computation, pulls fresher parameters to start over, so as to opportunistically improve the training quality. However, information sharing between clients is not allowed in FL.

\section{Heterogeneity Impact Study}
\label{sec:moti}

Compared with datacenter distributed learning and cross-silo FL, one of the key features of cross-device FL is the significant resource and data heterogeneity among clients, which can potentially 
impact both the training throughput and the model accuracy. Resource heterogeneity arises as a result of vast number of computational devices with varying computational and communication capabilities involved in the training process. The data heterogeneity arises as a result of two main reasons - (1) the varying number of training data samples available at each client and (2) the non-uniform distribution of classes and features among the clients.

\subsection{Formulating Vanilla Federated Learning}
\label{subsec:traditional_fl}

Cross-device FL is performed as an iterative process whereby the model is trained over a series of global training rounds, and the trained model is shared by all the involved clients. 
We define $K$ as the total pool of clients available to select from for each global training round, 
and $C$ as the set of clients selected per round. 
In every global training round, the aggregator selects a random fraction of clients $C_r$ from $K$.
The vanilla cross-device FL algorithm is briefly summarized in Alg.~\ref{alg:fedAvg}.
The aggregator first randomly initializes weights of the global model denoted by $\omega_{0}$. At the beginning of each round, the aggregator sends the current model weights to a subset of randomly selected clients. Each selected client then trains its local model with its local data
and sends back the updated weights to the aggregator after local training.
At each round, the aggregator waits until all selected clients respond with their corresponding trained weights. 
This iterative process keeps on updating the global model until a certain number of rounds are completed or a desired accuracy is reached.

The state-of-the-art cross-device FL system proposed in \cite{bonawitz2017practical} adopts a client selection policy where clients are selected randomly. 
A coordinator is responsible for creating and deploying a master aggregator and multiple child aggregators for achieving scalability as the real world cross-device FL system can involve up to $10^{10}$ clients~\cite{bonawitz2017practical, li2019federated, kairouz2019advances}. 
 
At each round, the master aggregator collects the weights from all the child aggreagtors to update the global model.

\subsection{Heterogeneity Impact Analysis}

The resource and data heterogeneity among involved clients may lead to varying response latencies (i.e., the time between a client receives the training task and returns the results) in the cross-device FL process, which is usually referred as the straggler problem.

We denote the response latency of a client $c_{i}$ as $L_{i}$, and the latency of a global training round is defined as
  \begin{equation}
  \scalebox{1.0}{
      $\label{latency}
      L_{r} = Max \Big(L_{1},L_{2},L_{3},L_{4}...L_{\mid C \mid}\Big)$}  .
  \end{equation}
where $L_{r}$ is the latency of round $r$. 
From Equation~\eqref{latency}, we can see the latency of a global training round is bounded by the maximum training latency of clients in $C$, i.e., the slowest client. 

We define $\tau$ levels of clients, i.e., within the same level, the clients have similar response latencies. 
Assume that the total number of levels is $m$ and  $\tau_{m}$ is the slowest level with $|\tau_{m}|$ clients inside. 
In the baseline case  (Alg.~\ref{alg:fedAvg}), the aggregator selects the clients randomly, resulting in a group of selected clients with composition spanning multiple client levels. 

We formulate the probability of selecting $|C|$ clients from all client levels except the slowest level $\tau_m$ as follows:
\begin{equation}
    Pr = \frac{\binom{|K|-|\tau_{m}|}{|C|}}{\binom{|K|}{|C|}} .
\end{equation} 
Accordingly, the probability of at least one client in $C$ comes from $\tau_m$ can be formulated as:
\begin{equation}
    Pr_s = 1 - Pr .
\end{equation}

\begin{theorem} 
\label{theo:first}
$ \frac{a-1}{b-1} < \frac{a}{b},\;while\; 1 < a < b $ 
\end{theorem}
\begin{proof}
Since $1 < a < b$, we could get $ab - b < ab - a$, that is $(a - 1)b < (b-1)a\; and\;  \frac{a-1}{b-1} < \frac{a}{b}$. 
\end{proof}
\begin{eqnarray}
    Pr_s &=& 1 - \frac{\binom{|K|-|\tau_{m}|}{|C|}}{\binom{|K|}{|C|}} \nonumber \\
    &=& 1 - \frac{(|K|-|\tau_{m}|)...(|K|-|\tau_{m}|-|C|+1)}{|K|...(|K|-|C|+1)} \nonumber \\
    &=& 1-
    \frac{|K|-|\tau_{m}|}{|K|}...\frac{|K|-|\tau_{m}|-|C|+1}{|K|-|C|+1} \nonumber \\
\end{eqnarray}

By applying theorem~\ref{theo:first}, we get:

\begin{eqnarray}
    Pr_s &>& 1 - \frac{|K|-|\tau_{m}|}{|K|}... \frac{|K|-|\tau_{m}|}{|K|} \nonumber \\
    &=& 1 -  {(\frac{|K|-|\tau_{m}|}{|K|})}^{|C|} \nonumber \\
\label{equ:final}
\end{eqnarray}

In real-world scenarios, large number of clients can be selected at each round, which makes $|K|$ extremely large. As a subset of $K$, the size of $C$ can also be sufficiently large. Since $\frac{|K|-|\tau_{m}|}{|K|} < 1$, we get $(\frac{|K|-|\tau_{m}|}{|K|})^{|C|} \approx 0$, which makes $Pr_s \approx 1$,  meaning in a vanilla cross-device FL training process, the probability of selecting at least one client from the slowest level is reasonably high for each round. According to Equation~\eqref{latency}, the  random selection strategy adopted by state-of-the-art cross-device FL system may suffer from a slow training performance.
 
\begin{algorithm}[H]
  \caption{Federated Averaging Training Algorithm}
  \label{alg:fedAvg}

\begin{algorithmic}[1]

  \STATE \textbf{Aggregator:} initialize weight $w_0$
  \FOR{each round $r = 0$ {\bfseries to} $N-1$}
    \STATE $C_r = $ (random set of $|C|$ clients)
    \FOR{each client c $\in$ $C_r$ \textbf{in parallel}}
    \STATE $w_{r+1}^c = TrainClient(c)$
    \STATE $s_c = $ (training size of c)
    \ENDFOR
    \STATE $w_{r+1} = \sum_{c=1}^{|C|} w_{r+1}^c * \frac{s_c}{\sum_{c=1}^{|C|} s_c}$
    
  \ENDFOR
\end{algorithmic}
\end{algorithm}
\subsection{Experimental Study}

To experimentally verify the above analysis and demonstrate the impact of resource heterogeneity and data quantity heterogeneity, we conduct a study with a setup similar to the paper~\cite{chai2019towards}. The testbed is briefly summarized as follows -
\begin{itemize}
    \item We use a total of 20 clients and each client is further divided into 5 groups with 4 client per group.
    \item We allocate 4 CPUs, 2 CPUs, 1 CPU, 1/3 CPU, 1/5 CPU for every client from group 1 through 5 respectively to emulate the resource heterogeneity.
    \item The model is trained on the image classification dataset CIFAR10~\cite{krizhevsky2014cifar} using the vanilla cross-device FL process ~\ref{subsec:traditional_fl} (model and learning parameters are detailed in Section \ref{sec:results}).
    \item Experiments with different data size for every client are conducted to produce data heterogeneity results.
\end{itemize}
 
As shown in Fig. \ref{fig:cifar10_characterization} (a), with the same amount of CPU resource, increasing the data size from 500 to 5000 results in a near-linear increase in training time per round. As the amount of CPU resources allocated to each client increases, the training time gets shorter. Additionally, the training time increases as the number of data points increase with the same number of CPUs. These preliminary results imply that the straggler issues can be severe under a complicated and heterogeneous FL environment.

To evaluate the impact of data distribution heterogeneity, we keep the same CPU resources for every client (i.e., 2 CPUs) and generate a biased class and feature distribution following~\cite{zhao2018federated}. 
Specifically, we distribute the dataset in such a way that every client has equal number of images from 2 (non-IID(2)), 5 (non-IID(5)) and 10 (non-IID(10)) classes, respectively. We train the model on Cifar10 dataset 
using the vanilla FL system as described in Section ~\ref{subsec:traditional_fl} with the model and training parameters detailed in Section ~\ref{sec:results}. 
As seen in Fig. \ref{fig:cifar10_characterization} (b), there is a clear difference in the accuracy with different non-IID distributions. The best accuracy is given by the IID since it represents a uniform class and feature distribution. As the number of classes per client is reduced, we observe a corresponding decrease in accuracy. Using 10 classes per client reduces the final accuracy by around 6\% compared to IID (it is worth noting that non-IID(10) is not the same as IID as the feature distribution in non-IID(10) is skewed compare to IID). In the case of 5 classes per client, the accuracy is further reduced by 8\%. The lowest accuracy is observed in the 2 classes per client case, which has a significant 18\% drop in accuracy. 

These studies demonstrate that the data and resource heterogeneity can cause significant impact on training time and training accuracy in cross-device FL. To tackle these problems, we propose \proj  --- a tier-based FL system which introduces a heterogeneity-aware client selection methodology that selects the most profitable  clients during each round of the training to minimize the heterogeneity impact while preserving the FL privacy proprieties, thus improving the overall training performance of cross-device FL (in the following, we use FL to denote cross-device FL for simplicity).

\begin{figure}[t]
     \centering
         \subfloat[Average training  \newline time per round  ]{\label{subfig:characterization_cifar10_training_time}\includegraphics[width=0.25\textwidth]{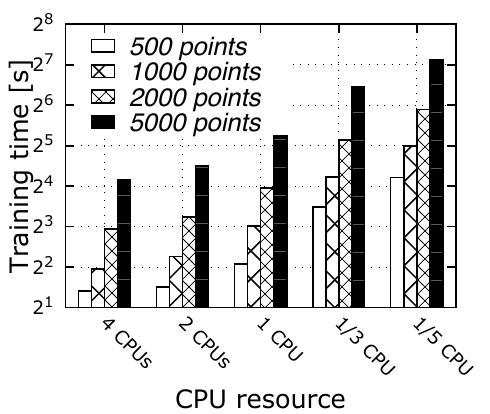}}
         \subfloat[Accuracy with varying class \newline distribution per client]{\label{subfig:characterization_cifar10_accuracy}
         \includegraphics[width=0.25\textwidth]{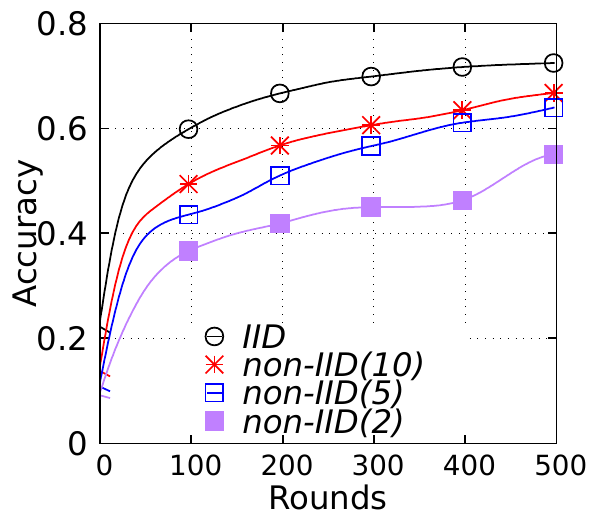}}
   \caption{(a) Training time per round (logscale) with resource heterogeneity and data amount heterogeneity; (b) accuracy under varying number of classes per client (non-IID).} 
    \label{fig:cifar10_characterization}
 \end{figure}
\section{{\proj}: A Tier-based Federated Learning System}
\label{sec:design}
\label{methodology}
In this section, we present the design of the proposed tier-based federated learning system \proj. 
The key idea of a tier-based system is that given the global training time of a round is bounded by the slowest client selected in that round (see Equation~\ref{latency}), selecting clients with similar response latency in each round can significantly reduce the training time.
We first give an overview of the architecture and the main flow of \proj system.
Then we introduce the profiling and tiering approach.
Based on the profiling and tiering results, we explain how a tier selection algorithm can potentially mitigate the heterogeneity impact through a straw-man proposal as well as the limitations of such static selection approach.
To this end, we propose an adaptive tier selection algorithm to address the limitations of the straw-man proposal.
Finally, we propose an analytical model through which one can estimate the expected training time using selection probabilities of tiers and the total number of training rounds.

\subsection{System Overview}
\label{subsec:tiering_algo}
The overall system architecture of {\proj} is present in Fig.~\ref{fig:Tiering_overview}.
{\proj} follows the system design to the state-of-the-art FL system~\cite{bonawitz2017practical} and adds two new components: a \emph{tiering module} (a profiler \& tiering algorithms) and a \emph{tier scheduler}. 
These newly added components can be incorporated into the coordinator of the existing FL system \cite{bonawitz2019towards}.  
It is worth to note that in Fig.~\ref{fig:Tiering_overview}, we only show a single aggregator rather than the hierarchical master-child aggregator design for a clean presentation purpose. {\proj} supports master-child aggregator design for scalability and fault tolerance.

In {\proj}, the first step is to collect the latency metrics of all the available clients through a lightweight profiling as detailed in Section \ref{subsec:tiering_algo}. 
The profiled data is further utilized by our tiering algorithm. This groups the clients into separate logical pools called \textit{tiers}. Once the scheduler has the tiering information (i.e., tiers that the clients belong to and the tiers' average response latencies), 
the training process begins. Different from vanilla FL that employs a random client selection policy, in {\proj} the scheduler selects a tier and then randomly selects targeted number of  clients from that tier.
After the selection of clients, the training proceeds as state-of-the-art FL system does. 
By design, {\proj} is non-intrusive and can be easily plugged into any existing FL system in that the tiering and scheduler module simply regulate client selection without intervening the underlying training process.

\begin{figure}[t]
    \centering
        {\label{subfig:Overview_tiering_FL}\includegraphics[width=0.48\textwidth]{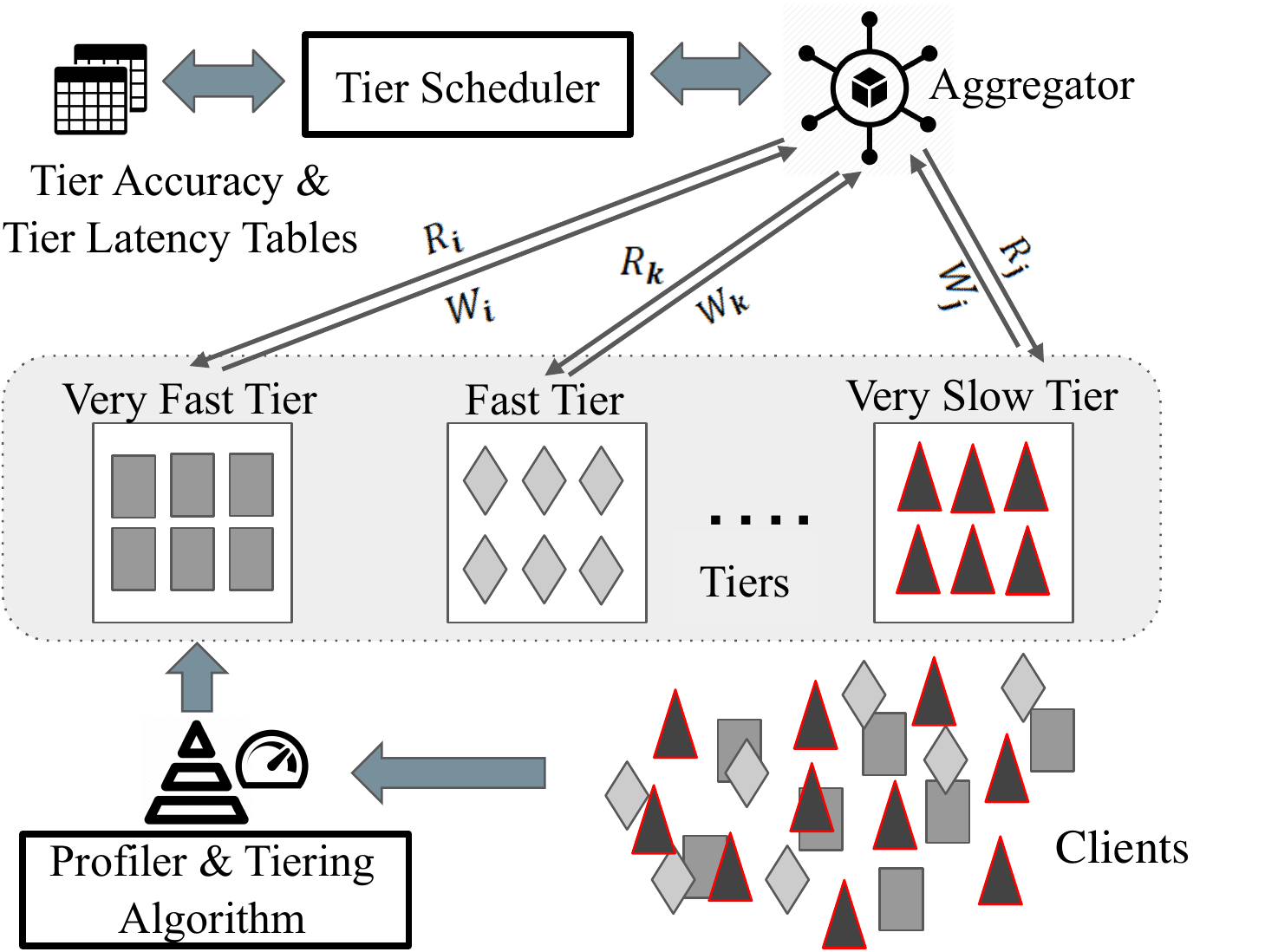}} 
    \caption{Overview of \proj.
    }
    \label{fig:Tiering_overview}
    \vspace{-0.2cm}
\end{figure}

\subsection{Profiling and Tiering}
\label{subsec:tiering_algo}
Given the global training time of a round is bounded by the slowest client selected in that round (see Equation~\ref{latency}), if we can select clients with similar response latency in each round, the training time can be improved.
However, in FL, the response latency is unknown a priori, which makes it challenging to carry out the above idea. 
To solve this challenge, we introduce a process through which the clients are \emph{tiered} (grouped) by the \textit{Profiling and Tiering} module as shown in Fig.~\ref{fig:Tiering_overview}. 
As the first step, all available clients are initialized with a response latency $L_i$ of 0. 
The profiling and tiering module then assigns all the available clients the profiling tasks.
The profiling tasks execute for $sync\_rounds$ rounds and in each profiling round, the aggregator asks every client to train on the local data and waits for their acknowledgement for $T_{max}$ seconds. 
All clients that respond within $T_{max}$ have their response latency value $RT_i$ incremented with the actual training time, while the ones that have timed out are incremented by $T_{max}$.
After $sync\_rounds$ rounds are completed, the clients with $L_i>=sync\_rounds*T_{max}$ are considered \emph{dropouts} and excluded from the rest of the calculation. 
The collected training latencies from clients creates a histogram,
which is split into $m$ groups and the clients that fall into the same group forms a tier.
The average response latency is then calculated for each group and recorded persistently which is used later for scheduling and selecting tiers.
The profiling and tiering can be conducted periodically for systems with changing computation and communication performance over the time so that clients can be adaptively grouped into the right tiers.

\subsection{Straw-man Proposal: Static Tier Selection Algorithm}
\label{subsecd:static_policy}
In this section, we present a naive static tier-based client selection policy and discuss its limitations, which motivates us to develop an advanced adaptive tier selection algorithm in the next section. 
While the profiling and tiering module introduced in Section~\ref{subsec:tiering_algo} groups clients into $m$ tiers based on response latencies, the tier selection algorithm focuses on how to select clients from the proper tiers in the FL process to improve the training performance. 
The natural way to improve training time is to prioritize towards faster tiers, rather than selecting clients randomly from all tiers (i.e., the full $K$ pool).
However, such selection approach reduces the training time without taking into consideration of the model accuracy and privacy properties.
To make the selection more general, one can specify each tier $n_{j}$ is selected based on a predefined probability, which sums to 1 across all tiers.
Within each tier, $\mid C \mid$ clients are uniform randomly selected.

In a real-world FL scenarios, there can be a large number of clients involved in the FL process (e.g., up to $10^{10}$)~\cite{bonawitz2017practical, li2019federated, kairouz2019advances}. Thus in our tiering-based approach, the number of tiers is set such that $m$ $<<$ $|K|$ and number of clients per tier $n_{j}$ is always greater than $|C|$. 
The selection probability of a tier is controllable, which results in different trade-offs.
If the users' objective is to reduce the overall training time, they may increase the chances of selecting the faster tiers. However, drawing  clients only from the fastest tier may inevitably introduce training bias due to the fact that different clients may own a diverse set of heterogeneous training data spread across different tiers; as a result, such bias may end up affecting the accuracy of the global model. To avoid such undesired behavior, it is preferable to involve clients from different tiers so as to cover a diverse set of training datasets. We perform an empirical analysis on the latency-accuracy trade-off in Section~\ref{sec:results}.
\vspace{0.3cm}
\subsection{Adaptive Tier Selection Algorithm}
\label{dynamic}

While the above naive static selection method is intuitive, it does not provide a method to automatically tune the trade-off to optimize the training performance nor adjust the selection based on changes in the system.
In this section, we propose an adaptive tier selection algorithm that can automatically strike a balance between training time and accuracy, and adapt the selection probabilities adaptively over training rounds based on the changing system conditions. 
\begin{algorithm}[th!]
  \caption{Adaptive Tier Selection Algorithm. $Credits_t$: the credits of Tier $t$, $I$: the interval of changing probabilities, $TestData_t$: evaluation dataset specific to that tier $t$, $A_t^r$: test accuracy of tier $t$ at round $r$, $\tau$: set of Tiers.}
  \label{alg:Dynamic}

\begin{algorithmic}[1]
  \STATE \textbf{Aggregator:} initialize weight $w_0$, $currentTier=1$, $TestData_{t}$, $Credits_{t}$, equal probability with $\frac{1}{T}$, for each tier $t$.
  
  \FOR{each round $r = 0$ {\bfseries to} $N-1$}
           \IF{$r\%I == 0$ and r $\geq$ I}
   \IF{$A_{currentTier}^r \leq A_{currentTier}^{r-I}$}
   \STATE $NewProbs = ChangeProbs(A_1^r,A_2^r...A_T^r)$
   \ENDIF
   \ENDIF
   \WHILE{$True$}
   \STATE $currentTier = $ (select one tier from T tiers with $NewProbs$)
   \IF{$Credits_{currentTier} > 0$}
   \STATE $Credits_{currentTier} = Credits_{currentTier}-1$
   \STATE \textbf{break}
   \ENDIF
   \ENDWHILE
    \STATE $C_r = $ (random set of $|C|$ clients from $currentTier$)
    \STATE $Credits_{currentTier} -= 1$
    \FOR{each client c $\in$ $C_r$ \textbf{in parallel}}
        \STATE $w_{r}^c = TrainClient(c)$
        \STATE $s_c = $ (training size of c)
    \ENDFOR
    \STATE $w_{r} = \sum_{c=1}^{|C|} w_{r+1}^c * \frac{s_c}{\sum_{c=1}^{|C|} s_c}$
    \FOR{each $t$ in $\tau$}
        \STATE $A^r_t = Eval(w_{r}, TestData_t)$
    \ENDFOR
  \ENDFOR
\end{algorithmic}
\end{algorithm}

The observation here is that heavily selecting certain tiers (e.g., faster tiers) may eventually lead to a biased model, {\proj} needs to balance the client selection from other tiers (e.g., slower tiers).
The question being which metric should be used to balance the selection.
Given the goal here is to minimize the bias of the trained model, we can monitor the accuracy of each tier throughout the training process.
A lower accuracy value of a tier $t$ typically indicates that the model has been trained with less involvement of this tier,
therefore tier $t$ should contribute more in the next training rounds. 
To achieve this, we can increase the selection probabilities for tiers with lower accuracy.
To achieve good training time, we also need to limit the selection of slower tiers across training rounds. 
Therefore, we introduce $Credits_t$, a constraint that defines how many times a certain tier can be selected. 

Specifically, a tier is initialized randomly with equal selection probability. After the weights are received and the global model is updated, 
the global model is evaluated on every client for every tier on their respective $TestData$ and their resulting accuracies are stored as the corresponding tier $t$'s accuracy for that round $r$. This is stored in $A^r_t$, which is the mean accuracy for all the clients in tier $t$ in training round $r$.
In the subsequent training rounds, 
the adaptive algorithm updates the probability of each tier based on that tier's test accuracy at every $I$ rounds. 
This is done in the function $ChangeProbs$, which adjusts the probabilities such that the lower accuracy tiers get higher probabilities to be selected for training; then with the new tier-wise selection probabilities ($NewProbs$), a tier which has remaining $Credits_t$ is selected from all available tiers $\tau$. 
The selected tier will have its $Credits_t$ decremented. As clients from a particular tier gets selected over and over throughout the training rounds, the $Credits_t$ for that tier ultimately reduces down to zero, meaning that it will not be selected again in the future. This serves as a control knob for the number of times a tier is selected and by setting this upper-bound, we can limit the amount of times a slower tier contributes to the training, thereby effectively gaining some control over setting a soft upper-bound on the total training time. For the straw-man implementation, we used a skewed probability of selection to manipulate training time. Since we now wish to adaptively change the probabilities, we add the $Credits_t$ to gain control over limiting training time.

On one hand, the tier-wise accuracy $A^t_r$ essentially makes {\proj}'s adaptive tier selection algorithm \emph{data heterogeneity aware}; as such, {\proj} makes the tier selection decision by taking into account the underlying dataset selection biasness, and automatically adapt the tier selection probabilities over time.
On the other hand, $Credits_t$ is introduced to intervene the training time by enforcing a constraint over the selection of the relatively slower tiers.
While $Credits_t$ and $A^r_t$ mechanisms optimize towards two different and sometimes contradictory objectives --- training time and accuracy, {\proj} cohesively synergizes the two mechanisms to strike a balance for the training time-accuracy trade-off. More importantly, with {\proj}, the decision making process is automated, thus relieving the users from intensive manual effort.
The adaptive algorithm is summarized in Algo.~\ref{alg:Dynamic}.

\subsection{Training Time Estimation Model}

In real-life scenarios, the training time and resource budget is typically finite. As a result, FL users may need to compromise between training time and accuracy. 
A training time estimation model would facilitate users to navigate the training time-accuracy trade-off curve to effectively achieve desired training goals.
Therefore, we build a training time estimation model that can estimate the overall training time based on the given latency values and the selection probability of each tier:
\begin{equation}
{L_{all}} =  \sum_{i=1}^{n} (L_{tier\_i} * P_{i}) * R  ,
\end{equation}
where $L_{all}$ is the total training time, $L_{tier\_i}$ is the response latency of tier $i$, $P_{i}$ is the probability of tier $i$, and $R$ is the total number of training rounds.
The model is a sum of products of the tier and latencies, which gives the latency expectation per round. This is multiplied by the total number of training rounds to get the total training time.

\vspace{0.5cm}
\subsection{Discussion: Compatibility with Privacy-Preserving Federated Learning}
FL has been used together with privacy preserving approaches such as differential privacy to prevent attacks  that aim to extract private information ~\cite{shokri2017membership,nasr2018comprehensive}. 
Privacy-preserving FL is based on client-level differential privacy, where the privacy guarantee is defined at each individual client. This can be accomplished by each client implementing a centralized private learning algorithm as their local training approach. For example, with neural networks this would be one or more epochs using the approach proposed in~\cite{abadi2016deep}. This requires each client to add the appropriate noise into their local learning to protect the privacy of their individual datasets. 
Here we demonstrate that \proj is compatible with such privacy preserving approaches.

Assume that for client $c_i$, one round of local training using a differentially private algorithm is ($\epsilon$, $\delta$)-differentially private, where $\epsilon$ bounds the impact any individual may have on the algorithm's output and $\delta$ defines the probability that this bound is violated. Smaller $\epsilon$ values therefore signify tighter bounds and a stronger privacy guarantee. Enforcing smaller values of $\epsilon$ requires more noise to be added to the model updates sent by clients to the FL server which leads to less accurate models. Selecting clients at each round of FL has distinct privacy and accuracy implications for client-level privacy-preserving FL approaches. For simplicity we assume that all clients are adhering to the same privacy budget and therefore same ($\epsilon$, $\delta$) values. Let us first consider the scenario wherein $C$ is chosen uniformly at random each round. Compared with each client participating in each round, the overall privacy guarantee, using random sampling amplification~\cite{beimel2010bounds}, improves from ($\epsilon$, $\delta$) to (O($q\epsilon$), $q\delta$) where $q = \frac{|C|}{|K|}$. This means that there is a stronger privacy guarantee \textit{with the same noise scale}. Clients may therefore add less noise per round or more rounds may be conducted without sacrificing privacy. For the tiered approach the guarantee also improves. Compared to ($\epsilon$, $\delta$) in the all client scenario, the tiered approach improves to an ($O(q_{max} \epsilon), q_{max} \delta$) privacy guarantee where the probability of selecting tier with weight $\theta_{j}$ is given by $\frac{1}{n_{tiers}} * \theta_j$, $q_{max} = \max_{j=1...|n_{tiers}|}q_j$ and $q_j = (\frac{1}{n_{tiers}} * \theta_j) \frac{|C|}{|n_j|}$. 

\vspace{1.0cm}
\section{Experimental Evaluation}
\label{sec:results}

We prototype \proj with both the naive and our adaptive selection approach and perform extensive testbed experiments under three scenarios: resource heterogeneity, data heterogeneity, and resource plus data heterogeneity.

\subsection{Experimental Setup}
\label{subsec:setup}

\noindent {\bf Testbed.}
As a proof of concept case study, we build a FL testbed for the syntehtic datasets by deploying 50 clients on a CPU cluster where each client has its own exclusive CPU(s) using Tensorflow ~\cite{abadi2016tensorflow}. 
In each training round, 5 clients are selected to train on their own data and send the trained weights to the server which aggregates them and updates the global model similar to \cite{mcmahan2016communication, bonawitz2019towards}. 
\cite{bonawitz2019towards} introduces multiple levels of server aggregators in order to achieve scalability and fault tolerance in extreme scale situations, i.e., with millions of clients. 
In our prototype, we simplify the system to use a powerful single aggragator as it is sufficient for our purpose here, i.e., our system does not suffer from scalabiltiy and fault tolerance issues, though multiple layers of aggregator can be easily integrated into \proj.

 We also extend the widely adopted large scale distributed FL framework LEAF  \cite{caldas2018leaf} in the same way. LEAF provides inherently non-IID with data quantity and class distributions heterogeneity. LEAF framework does not provide the resource heterogeneity among the clients, which is one of the key properties of any real-world FL system. The current implementation of the LEAF framework is a simulation of a FL system where the clients and server are running on the same machine. To incorporate the resource heterogeneity we first extend LEAF to support the distributed FL where every client and the aggregator can run on separate machines, making it a real distributed system. Next, we deploy the aggregator and clients on their own dedicated hardware. This resource assignment for every client is done through uniform random distribution resulting in equal number of clients per hardware type. By adding the resource heterogeneity and deploying them to separate hardware, each client mimics a real-world edge-device. Given that LEAF already provides non-IIDness, with the newly added resource heterogeneity feature the new framework provides a real world FL system which supports data quantity, quality and resource heterogeneity. For our setup, we use exactly the same sampling size used by the LEAF ~\cite{caldas2018leaf} paper (0.05) resulting in a total of 182 clients, each with a variety of image quantities.

\begin{table}[]
\caption{Scheduling Policy Configurations.}

\resizebox{\columnwidth}{!}{
\begin{tabular}{|c|c|c|c|c|c|c|}
\hline
DataSet                                                                        &      Policy          & \multicolumn{5}{c|}{Selection probabilities}       \\ \hline
\multirow{6}{*}{Cifar10 / FEMNIST}                                                      &   & Tier 1   & Tier 2   & Tier 3   & Tier 4   & Tier 5 \\ \cline{2-7} 
                                                                              & \textit{vanilla}       & N/A      & N/A      & N/A      & N/A      & N/A    \\
                                                                              & \textit{slow}       & $0.0$    & $0.0$    & $0.0$    & $0.0$    & $1.0$  \\
                                                                              & \textit{uniform}       & $0.2$    & $0.2$    & $0.2$    & $0.2$    & $0.2$  \\
                                                                              & \textit{random}       & $0.7$    & $0.1$    & $0.1$    & $0.05$   & $0.05$ \\
                                                                              & \textit{fast}       & $1.0$    & $0.0$    & $0.0$    & $0.0$    & $0.0$  \\
                                                                           \hline
                                                                         
\multirow{5}{*}{\begin{tabular}[c]{@{}c@{}}MNIST\\ FMNIST\end{tabular}}      & \textit{vanilla}      & N/A      & N/A      & N/A      & N/A      & N/A    \\
                                                                             & \textit{uniform}      & $0.2$    & $0.2$    & $0.2$    & $0.2$    & $0.2$  \\
                                                                              & \textit{fast1}      & $0.225$  & F$0.225$  & $0.225$  & $0.225$  & $0.1$  \\
                                                                              & \textit{fast2}      & $0.2375$ & $0.2375$ & $0.2375$ & $0.2375$ & $0.05$ \\
                                                                              & \textit{fast3}      & $0.25$    & $0.25$    & $0.25$    & $0.25$    & $0.0$  \\
 \hline
\end{tabular}
}
\label{tab:conf_selection_policy}
\end{table}

\begin{table}[t!]
   \centering
    \caption{Estimated VS Actual Training Time.}
    \scalebox{1}{
    \begin{tabular}{|c|c|c|c|}
        \hline
        Policy & Estimated [s] & Actual [s] & MAPE [\%]\\
        \hline
        \textit{slow} & $46242$ & $44977$ & $2.76$  \\
        \textit{uniform} & $12693$ & $12643$ & $0.4$   \\
        \textit{random} & $5143$ & $5053$ & $1.8$   \\
        \textit{fast} & $1837$ & $1750$ & $5.01$   \\
        \hline
    \end{tabular}
    }
    \label{tab:training_prediction}
\end{table}
\subsection{Experimental Results}

\noindent {\bf Models and Datasets.}
We use four image classification applications for evaluating \proj. We use {\it MNIST \footnote{http://yann.lecun.com/exdb/mnist/} and Fashion-MNIST \cite{xiao2017fashion}}, where each contains 60,000 training images and 10,000 test images, where each image is 28x28 pixels. We use a CNN model for both datasets, which starts with a 3x3 convolution layer with 32 channels and ReLu activation, followed by a 3x3 convolution layer with 64 channels and ReLu activation, a MaxPooling layer of size 2x2, a fully connected layer with 128 units and ReLu activation, and a fully connected layer with 10 units and ReLu activation. Dropout 0.25 is added after the MaxPooling layer, dropout 0.5 is added before the last fully connected layer. We use {\it Cifar10} \cite{krizhevsky2014cifar}, which contains richer features compared to MNIST and Fashion-MNIST. There is a total of 60,000 colour images, where each image has 32x32 pixels. The full dataset is split evenly between 10 classes, and partitioned into 50,000 training and 10,000 test images. The model is a four-layer convolution network ending with two fully-connected layers before the softmax layer. It was trained with a dropout of 0.25. Lastly we also use the FEMNIST data set from LEAF framework \cite{caldas2018leaf}. This is an image classification dataset which consists of 62 classes and the dataset is inherently non-IID with data quantity and class distributions heterogeneity. We use the standard model architecture as provided in LEAF \cite{caldas2018expanding}.

\begin{figure}[h!]
\label{}
    \centering
         
        \subfloat[Training time 500 rounds]{\label{subfig:cifar10_IID_training_time}\includegraphics[width=0.25\textwidth]{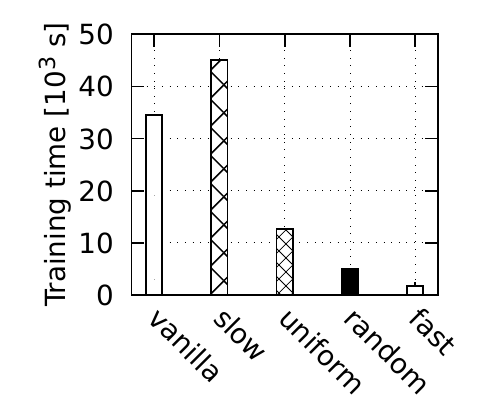}}
        \subfloat[Training time 500 rounds]{\label{subfig:cifar10IID_puredata_hetro_training_time}\includegraphics[width=0.25\textwidth]{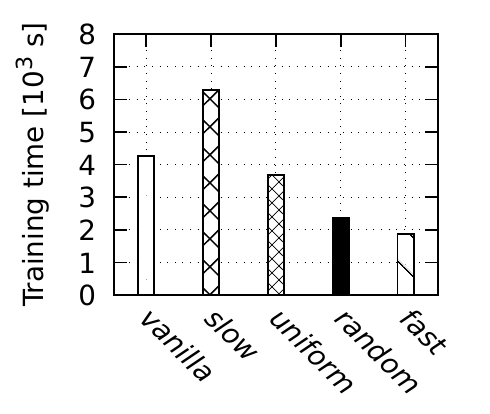}} 
        \\
        \subfloat[Accuracy over rounds]{\label{subfig:cifar10_IID_overall_accuracy}\includegraphics[width=0.25\textwidth]{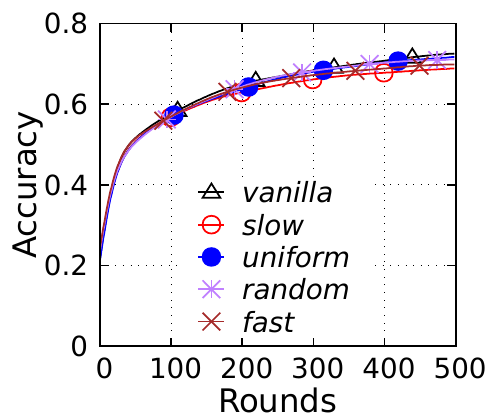}}
        \subfloat[Accuracy over round]{\label{subfig:cifar10IID_puredata_hetro_accuracy}\includegraphics[width=0.25\textwidth]{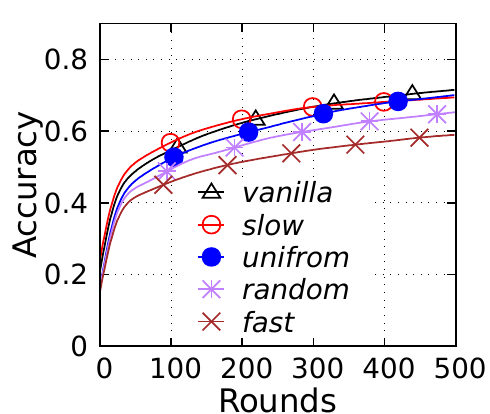}}
        \\
         \subfloat[Accuracy over time]{\label{subfig:cifar10_IID_overtime_accuracy}\includegraphics[width=0.25\textwidth]{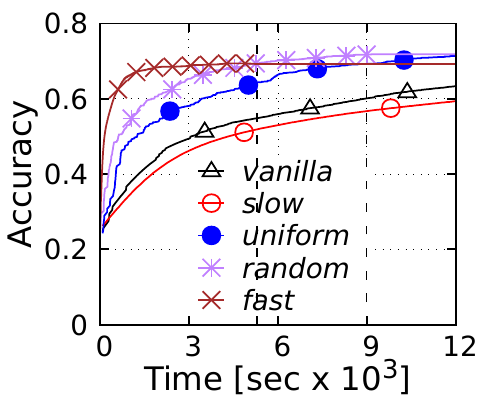}}
         \subfloat[Accuracy over time]{\label{subfig:cifar10IID_puredata_hetro_accuracy_overtime}\includegraphics[width=0.25\textwidth]{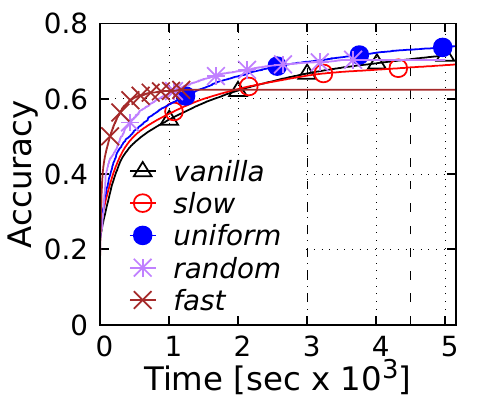}}
          \caption{Comparison results for different selection policies on Cifar10 with resource heterogeneity (Column 1) and data quantity heterogeneity (Column 2).}
    \label{fig:cifar10_IID_IID_resource_heterogeneity}
    
\end{figure}

\noindent {\bf Training Hyperparameters.}
We use RMSprop as the optimizer in local training and set the initial learning rate ($\eta$) as 0.01 and decay as 0.995. Local batch size of each client is 10, and local epochs is 1. The total number of clients ($|K|$) is $50$ and the number of participated clients ($|C|$) at each round is $5$. For FEMNIST we use the default training parameters provided by the LEAF Framework (SGD with lr 0.004, batch size 10). We train for a total of 2000 rounds for FEMNIST and 500 rounds for the synthetic datasets. Every experiment is run 5 times and we use the average values.

\begin{figure*}[t!]
  
    \centering
        \subfloat[\textit{vanilla}]{\label{subfig:cifar10_data_heterogeneity_training_time_nopolicy}\includegraphics[width=0.2\textwidth]{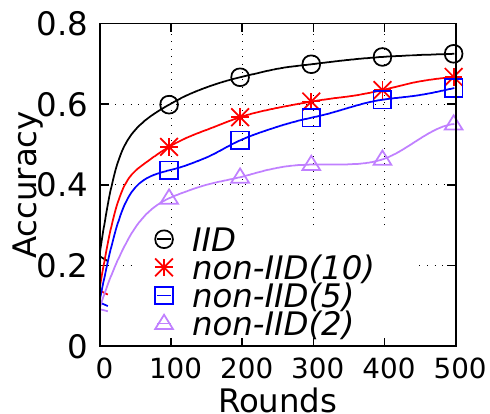}}
        \subfloat[\textit{slow} ]{\label{subfig:cifar10_data_heterogeneity_overtime_accuracy_00001}\includegraphics[width=0.2\textwidth]{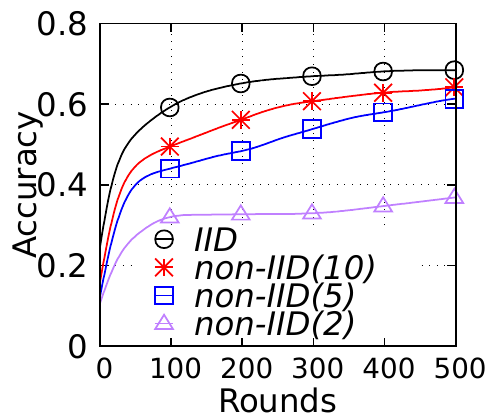}}
        \subfloat[
        \textit{uniform}]{\label{subfig:cifar10_data_heterogeneity_accuracy_22222}\includegraphics[width=0.2\textwidth]{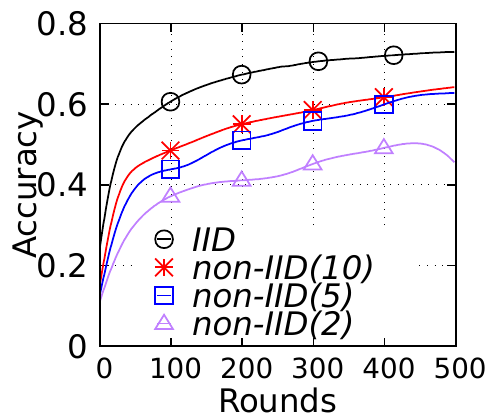}}
        \subfloat[\textit{random}]{\label{subfig:cifar10_data_heterogeneity_accuracy_71155}\includegraphics[width=0.2\textwidth]{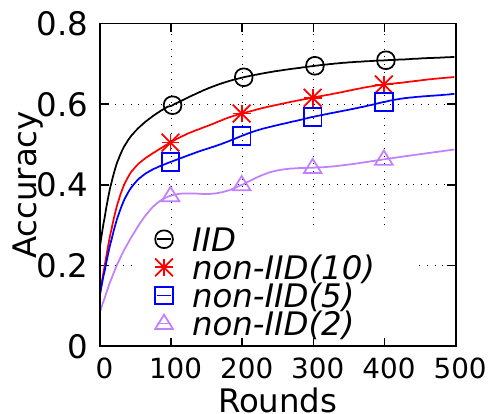}}
        \subfloat[\textit{fast}]{\label{subfig:cifar10_data_heterogeneity_accuracy_00001}\includegraphics[width=0.2\textwidth]{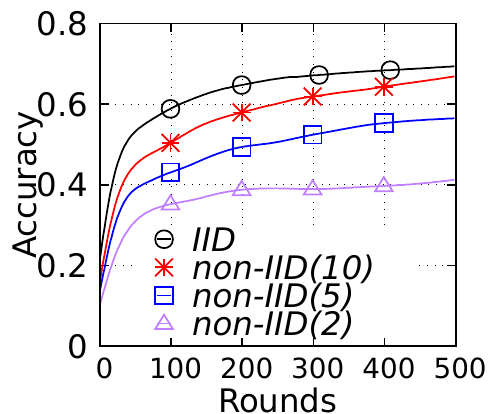}}
    \caption{Comparison results for different selection policies on Cifar10 with differnt levels of non-IID heterogeneity (Class) and fixed resources.}
    \label{fig:cifar10_sensitivity_analysis_accuracy}
    \vspace{-0.5cm}
\end{figure*}

\noindent {\bf Heterogeneous Resource Setup.}
Among all the clients, we split them into 5 groups with equal clients per group. 
For MNIST and Fashion-MNIST, each group is assigned with $2$ CPUs, $1$ CPU, $0.75$ CPU, $0.5$ CPU, and $0.25$ CPU per part respectively.
For the larger Cifar10 and FEMINIST model, each group is assigned with $4$ CPUs, $2$ CPUs, $1$ CPU, $0.5$ CPU, and $0.1$ CPU per part respectively. 
This leads to varying training time for clients belong to different groups.
By using the tiering algorithm of \proj, there are 5 tiers 

\noindent {\bf Heterogeneous Data Distribution.}
FL differs from the datacenter distributed learning in that the clients involved in the training process may have non-uniform data distribution in terms of amount of data per client and the non-IID data distribution.
$\bullet$ For {\it data quantity heterogeneity}, the training data sample distribution is 10\%, 15\%, 20\%, 25\%, 30\% of total dataset for difference groups, respectively, unless otherwise specifically defined. 
$\bullet$ For {\it non-IID heterogeneity}, 
we use different non-IID strategies for different datasets. For MNIST and Fashion-MNIST, we adopt the setting in \cite{mcmahan2016communication}, where we sort the labels by value first, divide into 100 shards evenly, and then assign each client two shards so that each client holds data samples from at most two classes. 
For Cifar10, we shard the dataset unevenly in a similar way and limit the number of classes to 5 per client (non-IID(5)) following \cite{zhao2018federated}, \cite{liu2019edge} unless explicitly mentioned otherwise. In the case of FEMINIST we use its default non-IID-ness.

\noindent {\bf Scheduling Policies.}
We evaluate several different naive scheduling policies of the proposed tier-based selection approach, defined by the selection probability from each tier, and compare it with the state-of-the-practice policy (or no policy) that existing FL works adopt, i.e., randomly select 5 clients from all clients in each round \cite{mcmahan2016communication, bonawitz2019towards}, agnostic to any heterogeneity in the system. We name it as \framebox{\textit{vanilla}}.
\framebox{\textit{fast}} is a policy that \proj only selects the fastest clients in each round. 
\framebox{\textit{random}} demonstrates the case where the selection of the fastest tier is prioritized over slower ones. 
\framebox{\textit{uniform}} is a base case for our tier-based naive selection policy where every tier has an equal probability of being selected. 
\framebox{\textit{slow}} is the worst policy that \proj only selects clients from the slowest tiers and we only include it here for reference purpose so that we can see a performance range between the best case and the worst case scenarios for static tier-based selection approach.
We use the above policies for CIFAR-10 and FEMINIST training.
For MNIST and Fashion-MNIST, given it is a much more lightweight workload, we focus on demonstrating the sensitivity analysis when the policy prioritizes more aggressively towards the fast tier, i.e., from \framebox{\textit{fast1}} to \framebox{\textit{fast3}}, the slowest tier's selection probability has reduced from 0.1 to 0 while all other tiers got equal probability.
We also include the \framebox{\textit{uniform}} policy for comparison, which is the same as in CIFAR-10. 
Table~\ref{tab:conf_selection_policy} summarizes all these scheduling policies by showing their selection probabilities.

\subsubsection{Training Time Estimation via Analytical Model} 

In this section, we evaluate the accuracy of our training time estimation model on different naive tier selection policies by comparing the estimation results of the model with the measurements obtained from test-bed experiments. The estimation model takes as input of the profiled average latency of each tier, the selection probabilities, and total number of training rounds to estimate the training time. 
We use mean average prediction error (MAPE) as the evaluation metric, which is defined as follows:

\begin{figure}[h!]
    \centering
        \subfloat[Training time 500 rounds]{\label{subfig:Mnist_nonIID_time}\includegraphics[width=0.25\textwidth]{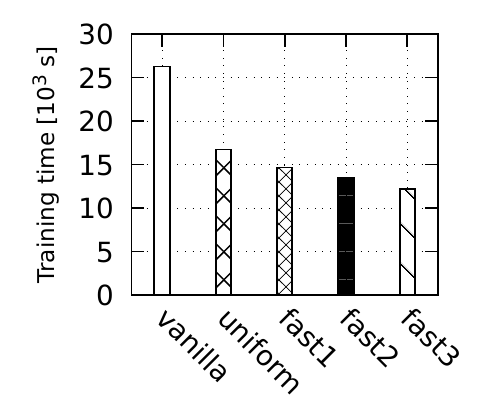}}
        \subfloat[Training time 500 rounds]{\label{subfig:FMnist_nonIID_time}\includegraphics[width=0.25\textwidth]{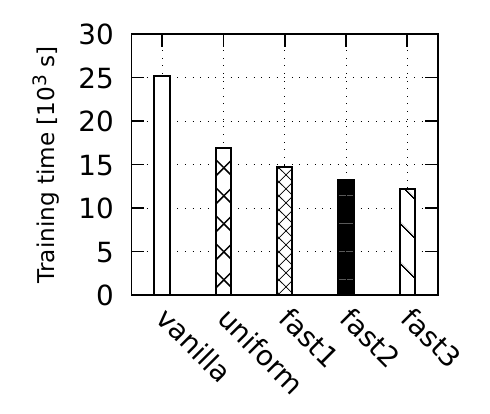}}\\
        \subfloat[Accuracy over round]{\label{subfig:Mnist_nonIID_Accuracy}\includegraphics[width=0.25\textwidth]{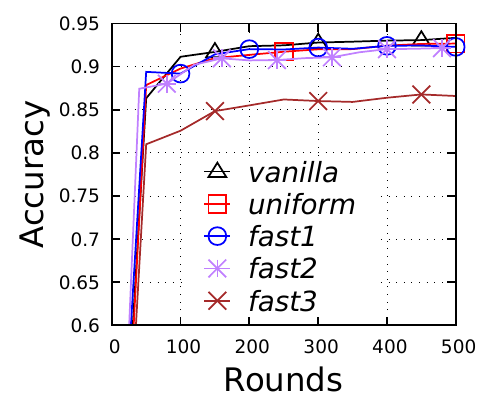}}
        \subfloat[Accuracy over round]{\label{subfig:FMnist_nonIID_Accuracy}\includegraphics[width=0.25\textwidth]{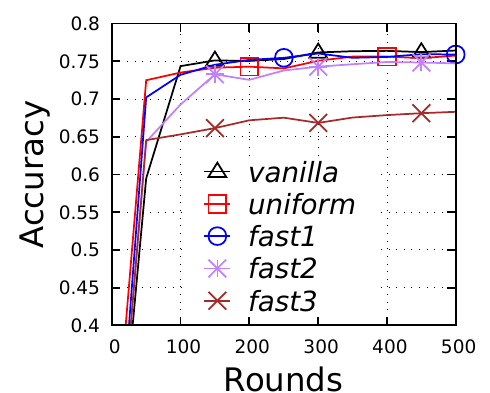}}
        \\
    \caption{Comparison results for different selection policies on MNIST (Column 1) and FMNIST (Column 2) with resource plus data heterogeneity.}
    \label{mnist:nonIID}
    \vspace{-0.3cm}
\end{figure}
\begin{equation}
  \text{MAPE} = \frac{|L_{all}^{est} - L_{all}^{act}|}{L_{all}^{act}} * 100 ,
\end{equation}

where $L_{all}^{est}$ is the estimated training time calculated by the estimation model and $L_{all}^{act}$ is the actual training time measured during the training process. 
Table~\ref{tab:training_prediction} demonstrates the comparison results. 
The results suggest the analytical model is very accurate as the estimation error never exceeds more than 6 \%.

\subsubsection{Resource Heterogeneity}
In this sections, we evaluate the performance of \proj in terms of training time and model accuracy in a resource heterogeneous environment as depicted in~\ref{subsec:setup} and we assume there is no data heterogeneity. 
In practice, data heterogeneity is a norm in FL, we evaluate this scenario to demonstrate how \proj tame resource heterogeneity alone and we evaluate the scenario with both resource and data heterogeneity in Section~\ref{resoure_plus_data}.

In the interest of space, we only present the Cifar10 results here as MNIST and Fashion-MNIST share the similar observations.
The results are organized in Fig.~\ref{fig:cifar10_IID_IID_resource_heterogeneity} (column 1), which clearly indicate that when we prioritize towards the fast tiers, the training time reduces significantly. 
Compared with \textit{vanilla}, \textit{fast} achieves almost 11 times improvement in training time, see Fig.~\ref{fig:cifar10_IID_IID_resource_heterogeneity} (a).
One interesting observation is that even \textit{uniform} has an improvement of over 6 times over the \textit{vanilla}. 
This is because the training time is always bounded by the slowest client selected in each training round. 
In \proj, selecting clients from the same tier minimizes the straggler issue in each round, and thus greatly improves the training time. 
For accuracy comparison, Fig.~\ref{fig:cifar10_IID_IID_resource_heterogeneity} (c) shows that the difference between polices are very small, i.e., less than 3.71\% after 500 rounds.
However, if we look at the accuracy over wall-clock time, \proj achieves much better accuracy compared to \textit{vanilla}, i.e., up to 6.19\% better if training time is constraint, thanks to the much faster per round training time brought by \proj, see Fig.~\ref{fig:cifar10_IID_IID_resource_heterogeneity} (e). Note here that different policies may take very different amount of wall-clock time to finish 500 rounds.

\begin{figure}[t]
    \centering
        \subfloat[Training time 500 rounds]{\label{subfig:cifar10_nonIID_training_time}\includegraphics[width=0.25\textwidth]{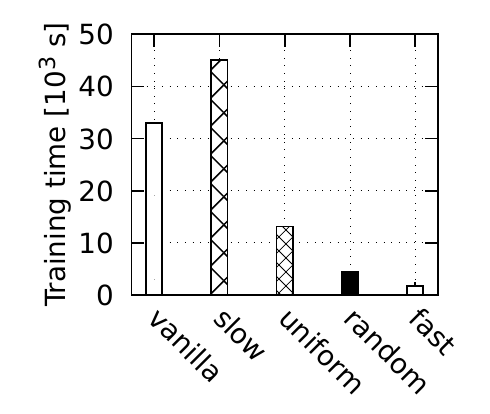}}
        \subfloat[Training time 500 rounds]{\label{subfig:cifar10_NonIID_resource_plus_data_hetrogenity_overall_training_time}\includegraphics[width=0.25\textwidth]{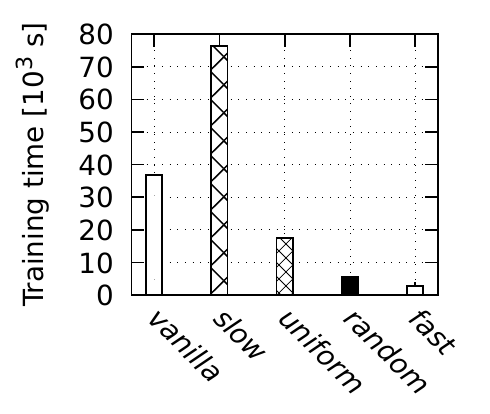}}\\
       \subfloat[Accuracy over rounds]{\label{subfig:cifar10_nonIID_overall_accuracy}\includegraphics[width=0.25\textwidth]{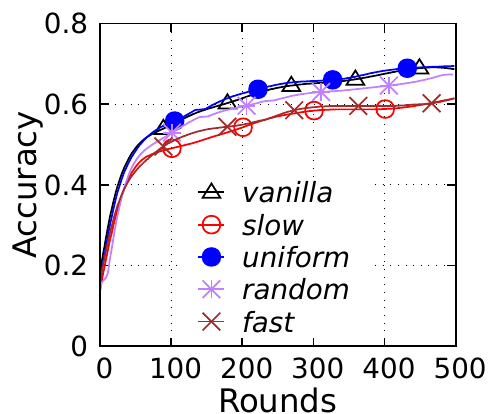}}
       \subfloat[Accuracy over rounds]{\label{subfig:cifar10_NonIID_resource_plus_data_hetrogenity_overall_accuracy}\includegraphics[width=0.25\textwidth]{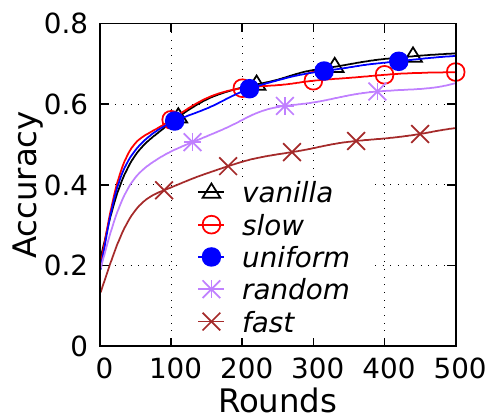}} \\
        \subfloat[Accuracy over time]{\label{subfig:cifar10_nonIID_overtime_accuracy}\includegraphics[width=0.25\textwidth]{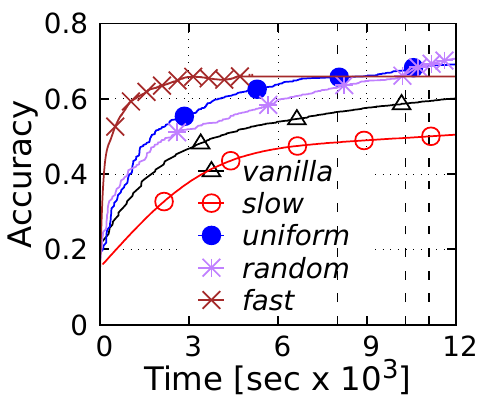}}
        \subfloat[Accuracy over time ]{\label{subfig:cifar10_NonIID_resource_plus_datahetrogenity_oertime_accuracy}\includegraphics[width=0.25\textwidth]{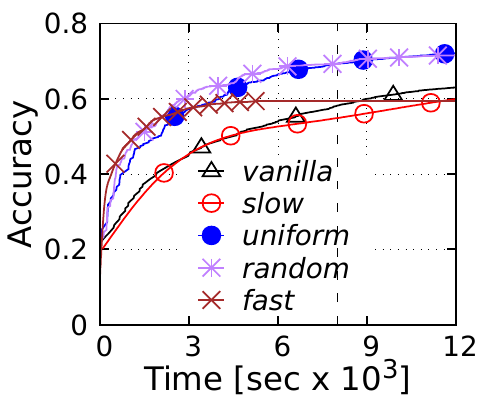}}
        \caption{Comparison results for different selection policies on Cifar10 with resource plus non-IID heterogeneity heterogeneity (Column 1) and resource, data quantity, and non-IID heterogeneity heterogeneity (Column 2).}
        \vspace{-4mm}
    \label{fig:cifar10_nonIID_resource_hetro}
\end{figure}

\subsubsection{Data Heterogeneity}
\label{data_heterogeneity}
 
In this section, we evaluate data heterogeneity due to both {\it data quantity heterogeneity} and {\it non-IID heterogeneity} as depicted in Section~\ref{subsec:setup}.
To demonstrate only the impact from data heterogeneity, we allocate homogeneous resource to each client, i.e., 2 CPUs per client.

\noindent $\bullet$ {\it Data quantity heterogeneity.} The training time and accuracy results are show in Fig.~\ref{fig:cifar10_IID_IID_resource_heterogeneity} (column 2).
In the interest of space, we only show Cifar10 results here.
From the training time comparison in Fig.~\ref{fig:cifar10_IID_IID_resource_heterogeneity} (b), it is interesting that \proj also helps in data heterogeneity only case and achieves up to 3 times speedup. The reason is that {\it data quantity heterogeneity} may also result in different round time, which shares the similar effect as resource heterogeneity. Fig.~\ref{fig:cifar10_IID_IID_resource_heterogeneity} (d) and (f) show the accuracy comparison, where we can see \textit{fast} has relatively obvious drop compared to others because Tier 1 only contains 10\% of the data, which is a significant reduction in volume of the training data. \textit{slow} is also a heavily biased policy towards only one tier, but Tier 5 contains 30\% of the data thus \textit{slow} maintains good accuracy while worst training time.
These results imply that like resource heterogeneity only, data heterogeneity only can also benefit from \proj. However, policies that are too aggressive toward faster tier needs to be used very carefully as clients in fast tier achieve faster round time due to using less samples. 
It is also worth to point out that in our experiments the total amount of data is relatively limited. In a practical case where data is significantly more, the accuracy drop of \textit{fast} is expected to be less pronounced.

\noindent $\bullet$ {\it non-IID heterogeneity.}
We observe that non-IID heterogeneity does not impact the training time. Hence, we omit the results here. However, non-IID heterogeneity effects the accuracy.
Fig. \ref{fig:cifar10_sensitivity_analysis_accuracy} shows the accuracy over rounds given 2, 5, and 10 classes per client in a non-IID setting. We also show the IID results in plot for comparison.
These results show that as the heterogeneity level in non-IID heterogeneity increases, the accuracy impact also increases for all policies due to the strongly biased training data.
Another important observation is that \textit{vanilla} case and \textit{uniform} have a better resilience than other policies, thanks to the unbiased selection behavior, which helps minimize further bias introduced during the client selection process.

\begin{figure}[t]
    \centering
        \subfloat[Training time for 500 rounds]{\label{subfig:cifar10_noniid_dynamic_overall_training_time}\includegraphics[width=0.25\textwidth]{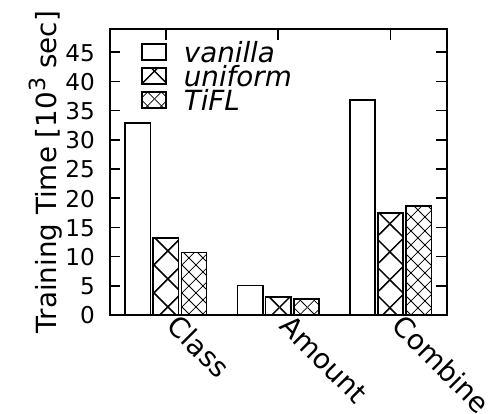}} 
        \subfloat[Accuracy at 500 rounds]{\label{subfig:cifar10_noniid_dynamic_overall_accuracy}\includegraphics[width=0.25\textwidth]{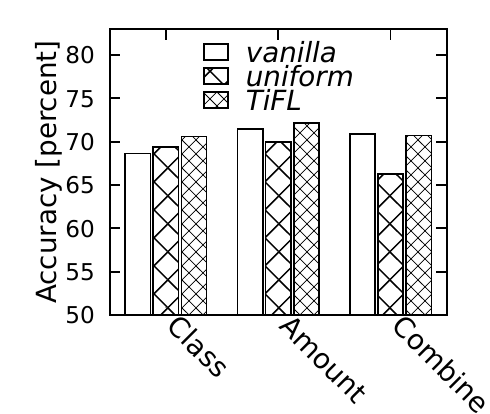}} \\
    \caption{Comparison results for different selection policies on Cifar10 with data quantity heterogeneity (Amount), non-IID heterogeneity (Class), and resource plus data heterogeneity (Combine).} 
    \label{fig:cifar10_noniid_dynamic}
\end{figure}
\subsubsection{Resource and Data Heterogeneity}
\label{resoure_plus_data}
This section presents the most practical case study, since here we evaluate with both resource and data heterogeneity combined. 

\textbf{MNIST} and \textbf{Fashion-MNIST (FMNIST)} results are shown in Fig.~\ref{mnist:nonIID} columns 1 and 2 respectively.
Overall, policies that are more aggressive towards the fast tiers bring more speedup in training time.
For accuracy, all polices of \proj are close to \textit{vanilla}, except \textit{fast3} falls short as it completely ignores the data in Tier 5.

\begin{figure*}[t!]
 \vspace{-0.3cm}
    \centering
        \subfloat[2-class per client]{\label{subfig:cifar10_data_heterogeneity_accuracy_2classes}\includegraphics[width=0.25\textwidth]{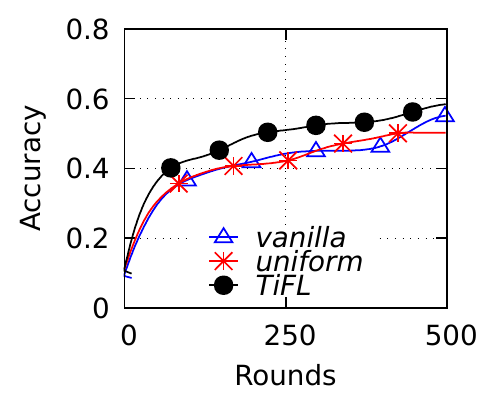}}
        \subfloat[5-class per client]{\label{subfig:cifar10_data_heterogeneity_accuracy_5classes}\includegraphics[width=0.25\textwidth]{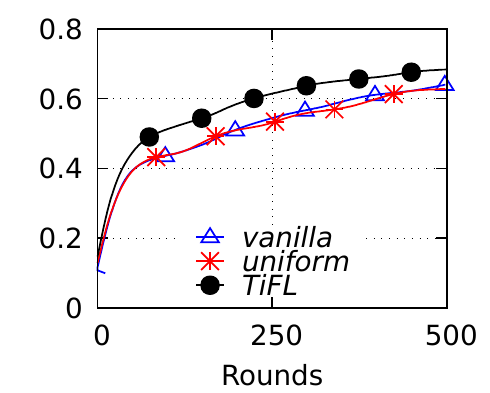}}
        \subfloat[10-class per client]{\label{subfig:cifar10_data_heterogeneity_accuracy_10classes}\includegraphics[width=0.25\textwidth]{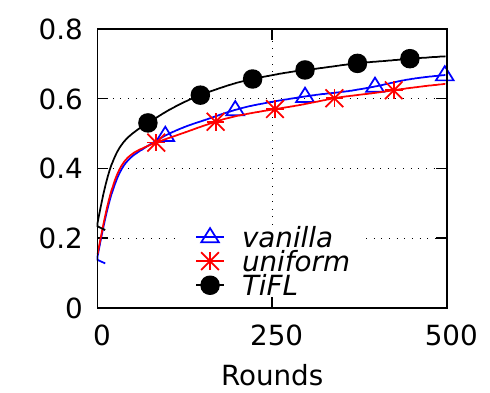}}
    \caption{Comparison results of Cifar10 under non-IID heterogeneity (Class) for different client selection policies with fixed resources (2 CPUs) per client.} 
    \label{fig:cifar10_dynamic_data_heterogeneity_accuracy}
\end{figure*}

\begin{figure}[h!]
\vspace{-.3cm}
    \centering
        \subfloat[Training time for 2000  rounds]{\label{subfig:leaf_noniid_dynamic_overall_training_time}\includegraphics[width=0.25\textwidth]{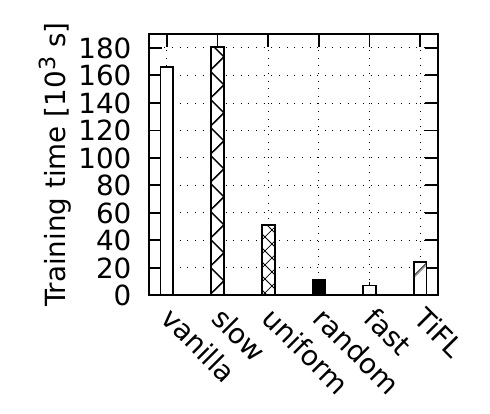}} 
        \subfloat[Accuracy over rounds]{\label{subfig:leaf_noniid_dynamic_overall_accuracy}\includegraphics[width=0.25\textwidth]{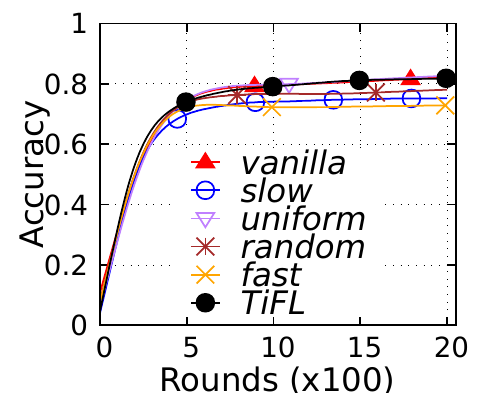}} \\
    \caption{Comparison results for different selection policies on LEAF with default data heterogeneity (quantity, non-IID heterogeneity), and resource heterogeneity.} 
    \label{fig:leaf_noniid_dynamic}
    \vspace{-.2cm}
\end{figure}

\textbf{Cifar10} results are shown in Fig.~\ref{fig:cifar10_nonIID_resource_hetro} column 1. It presents the case of resource heterogeneity plus non-IID data heterogeneity with equal data quantities per client and the results are similar to resource heterogeneity only since non-IID data with the same amount of data quantity per client results in a similar effect of resource heterogeneity in terms of training time. However, the accuracy degrades slightly more here as because of the non-IID-ness the features are skewed, which results in more training bias among different classes.

Fig.~\ref{fig:cifar10_nonIID_resource_hetro} column 2 shows the case of resource heterogeneity plus both the data quantity heterogeneity and non-IID heterogeneity. 
As expected, the training time shown in Fig.~\ref{fig:cifar10_nonIID_resource_hetro} (b) is similar to Fig.~\ref{fig:cifar10_nonIID_resource_hetro} (a) since the training time impact from different data amounts can be corrected by \proj.
However, the behaviors of round accuracy are quite different here as shown in Fig.~\ref{fig:cifar10_nonIID_resource_hetro} (d). The accuracy of \textit{fast} has degraded a lot more due to the data quantity heterogeneity as it further amplifies the training class bias (i.e., the data of some classes become very little to none) in the already very biased data distribution caused by the non-IID heterogeneity. 
Similar reasons can explain for other policies
The best performing policy in accuracy here is the \textit{uniform} case and is almost the same as \textit{vanilla}, thanks to the even selection nature which results in little increase in training class bias.
 Fig.~\ref{fig:cifar10_nonIID_resource_hetro} (f) shows the wall-clock time accuracy. As expected, the significantly improved per round time in \proj shows its advantage here as within the same time budget, more iterations can be done with shorter round time and thus remedies the accuracy disadvantage per round.
\textit{fast} still falls short than \textit{vanilla} in the long run as the limited and biased data limits the benefits of more iterations. 
\textit{fast} also perform worse than \textit{vanilla} as it has no training advantage.

\subsubsection{Adaptive Selection Policy}\label{dynamic_evl}
The above evaluation demonstrate the naive selection approach in \proj can significantly improve the training time, but sometimes can fall short in accuracy, especially when strong data heterogeneity presents as such approach is data-heterogeneity agnostic.   
In this section, we evaluate the proposed \textit{adaptive} tier selection approach of \proj, which takes into consideration of both resource and data heterogeneity when making scheduling decisions without privacy violation.
We compare \textit{adaptive} with \textit{vanilla} and \textit{uniform}, and the later is the best accuracy performing static policy.

Fig.~\ref{fig:cifar10_noniid_dynamic} shows \textit{adaptive} outperforms \textit{vanilla} and \textit{uniform} in both training time and accuracy for resource heterogeneity with data quantity heterogeneity (Amount) and non-IID heterogeneity (Class), thanks to the data heterogeneity-aware schemes.
In the combined resource and data heterogeneity case (Combine), \textit{adaptive} achieves comparable accuracy with \textit{vanilla} with almost half of the training time, and performs similar as \textit{uniform} in training time while improves significantly in accuracy.
The above robust performance of \textit{adaptive} is credited to both the resource and data heterogeneity-aware schemes.
To demonstrate the robustness of \textit{adaptive}, we compare the accuracy over rounds for different policies under different non-IID heterogeneity in Fig. \ref{fig:cifar10_dynamic_data_heterogeneity_accuracy}.
It is clear that \textit{adaptive} consistently outperforms \textit{vanilla} and \textit{uniform} in different level of non-IID heterogeneity.

\subsubsection{Adaptive Selection Policy(LEAF)}

This section provides the evaluation of \proj using a widely adopted large scale distributed FL dataset FEMINIST from the  LEAF framework \cite{caldas2018leaf} . We use exactly the same configurations (data distribution, total number of clients, model and training hyperparameters) as mentioned in \cite{caldas2018leaf} resulting in total number of 182 clients, i.e. deploy-able edge devices. Since LEAF provides it's own data distribution among devices the addition of resource heterogeneity results in a range of training times thus generating a scenario where every edge device has a different training latency. We further incorporated \proj's tiering module and selection policy to the extended LEAF framework. The profiling modules collects the training latency of each clients and creates a logical pool of tiers which is further utilized by the scheduler. The scheduler selects a tier and then the edge clients within the tier in each training round. For our experiments with LEAF we limit the total number of tiers to 5 and during each round we select 10 clients, with 1 local epoch per round.

Figure \ref{fig:leaf_noniid_dynamic} shows the training time and accuracy over rounds for LEAF with different client selection policies. Figure \ref{subfig:leaf_noniid_dynamic_overall_training_time} shows the training time for different selection policies. The least training time is achieved by using the \textit{fast} selection policy  however, it impact the final model accuracy by almost 10\% compared to \textit{vanilla} selection policy. The reason for the least accuracy for \textit{fast} is the result of less training point among the clients in tier 1. One interesting observation is \textit{slow} out performs the selection policy \textit{fast} in terms of accuracy even though each of these selection policies rely on data from only one tier. It must be noted that the slow tier is not only the reason of less computing resources but also the higher quantity of training data points. These results are consistent with our observations from the results presented in Section \ref{data_heterogeneity}.

Figure \ref{subfig:leaf_noniid_dynamic_overall_accuracy} shows the accuracy over-rounds for different selection policies. Our proposed \textit{adaptive} selection policy achieves 82.1\% accuracy and outperforms the \textit{slow} and \textit{fast} selection policies by 7\% and 10\% respectively. The \textit{adaptive} policy is on par with the \textit{vanilla} and \textit{uniform} ( 82.4\% and 82.6\% respectively). when comparing the total training time for 2000 rounds \textit{adaptive} achieves 7 $\times$ and 2 $\times$ improvement compare to \textit{vani} and \textit{uniform} respectively. \textit{fast} and \textit{random} both outperformed the \textit{adaptive} in terms of training time however, even after convergence the accuracy for both of these selection policies show a noticeable impact on the final model accuracy. The results for FEMINIST using the extended LEAF framework for both accuracy as well as training time are also consistent with the results reported in Section \ref{dynamic_evl}.

\section{Conclusion}
\label{sec:conclusion}
In this paper, we investigate and quantify the heterogeneity impact on ``decentralized virtual supercomputer'' - FL systems. Based on the observations of our case study, we propose and prototype a Tier-based Federated Learning System called \proj.
Tackling the resource and data heterogeneity, {\proj} employs a tier-based approach that groups clients in tiers by their training response latencies and selects clients from the same tier in each training round. 
To address the challenge that data heterogeneity information cannot be directly measured due to the privacy constraints, we further design an \textit{adaptive} tier selection approach that enables \proj be data heterogeneity aware and outperform conventional FL in various heterogeneous scenarios: \textit{resource heterogeneity}, \textit{data quantity heterogeneity}, \textit{non-IID data heterogeneity}, and their combinations. Specifically, \proj achieves an improvement over conventional FL by up to 3$\times$ speedup in overall training time and by 6\% in accuracy.

\bibliographystyle{ACM-Reference-Format}
\bibliography{main}

\end{document}